\documentclass[11pt]{article}

\usepackage[margin=1in]{geometry}
\usepackage{CJK, hyperref, amsmath, amsthm}

\newtheorem{assumption}{Assumption}
\newtheorem{corollary}{Corollary}
\newtheorem{lemma}{Lemma}
\newtheorem{theorem}{Theorem}

\DeclareMathOperator{\diag}{diag}
\DeclareMathOperator{\poly}{poly}
\DeclareMathOperator{\rank}{rank}
\DeclareMathOperator{\tr}{tr}

\begin{document}

\begin{CJK*}{UTF8}{}

\title{Provably efficient neural network representation for image classification}

\CJKfamily{gbsn}

\author{Yichen Huang (黄溢辰)\\
California Institute of Technology\\
Pasadena, California 91125, USA\\
ychuang@caltech.edu
}

\maketitle

\end{CJK*}

\begin{abstract}

The state-of-the-art approaches for image classification are based on neural networks. Mathematically, the task of classifying images is equivalent to finding the function that maps an image to the label it is associated with. To rigorously establish the success of neural network methods, we should first prove that the function has an efficient neural network representation, and then design provably efficient training algorithms to find such a representation. Here, we achieve the first goal based on a set of assumptions about the patterns in the images. The validity of these assumptions is very intuitive in many image classification problems, including but not limited to, recognizing handwritten digits.

\end{abstract}

\section{Introduction}

Despite the fact that neural network methods are highly successful in machine learning, a theoretical understanding of the practical success remains to be elucidated. We make rigorous progress towards this goal for image classification. The setting is as follows.

Suppose each image has $n\times n$ pixels, and is black and white. This is without loss of generality, as we would obtain the same result for colorful images, provided that the number of colors is $O(1)$. There are altogether $2^{n^2}$ images, each of which is represented by an element in $\{0,1\}^{\otimes n^2}$. For each label, we define a function $f:\{0,1\}^{\otimes n^2}\to\{0,1\}$ such that $f(x)=1$ if and only if the label matches the image $x$. (The function $f$ should carry the label as a subscript, which is omitted for simplicity.) In practice, $f$ is sparse because most images are featureless and do not have a label. However, there can be an exponential (in $n$) number of images having labels.

In supervised learning, we are given a training set $T\subset\{0,1\}^{\otimes n^2}$ and the value $f(x)$ for each $x$ in $T$. The task is to find $f$. The standard approach is to first presume that $f$ is well approximated by a neural network with $\poly n$ parameters, and then tune the parameters to minimize the difference between $f$ and the function represented by the neural network. To rigorously establish the success of the method, we should first prove that $f$ has an efficient neural network representation, and then design provably $\poly n$-time training algorithms to find a sufficiently good approximation to $f$ \cite{LTR17}.

In this paper, we achieve the first goal (``efficient representation'') based on a set of assumptions about the patterns in the images. The validity of these assumptions is very intuitive in many image classification problems. As an example, it will be justified in the context of recognizing handwritten digits. It should be clear that assumptions are necessary to exclude the possibility that $f^{-1}(1)$ is a random subset of $\{0,1\}^{\otimes n^2}$, in which case an efficient representation of $f$ is impossible.

The second goal (``efficient training'') is more important and challenging. Besides algorithmic advances, one has to define how the training set $T$ is obtained, and then prove that a neural network trained on $T$ can generalize well. We leave it to future work.

\section{Assumptions} \label{s2}

In this section, we introduce two assumptions about the patterns in the images, and justify them in the context of recognizing handwritten digits. Intuitively, it should be clear that both assumptions are valid in many other image classification problems.

We briefly describe some characteristic features of handwritten digits. An image that represents a digit is a black curve on a white background. We assume that the linewidth is $O(1)$ pixels on the whole curve, but different part of the curve may have different linewidth. Of course, a reasonable amount of distortion is allowed. Unlike the MNIST database \cite{LBBH98}, we do not assume that the digits are size-normalized or centered. For example, you may write a small ``$8$'' at the top right corner.

For each image $x\in\{0,1\}^{\otimes n^2}$, let $x=\overline{x_1x_2\cdots x_n}$, where the overline denotes string concatenation, and $x_i\in\{0,1\}^{\otimes n}$ is the configuration of $x$ in the $i$th row. The first assumption states that the total number of row configurations for all images having labels is upper bounded by a polynomial in $n$.

\begin{assumption} \label{a1}
\begin{equation}
|\{x_i:f(x)=1\}|=\poly n,\quad\forall i.
\end{equation}
\end{assumption}

\begin{proof}[Justification of Assumption \ref{a1}]
Consider, for example, the images representing the digit ``$0$.'' There are three cases:

\emph{Case 1}. In a row above or below the digit, the pixels are all white.

\emph{Case 2}. In a row tangent to the curve, there is one contiguous region of black pixels.

\emph{Case 3}. If a row divides ``$0$'' into two parts, then there are two connected components, each of which contains $O(1)$ black pixels.

Taking into account all these cases, the number of row configurations for all images representing ``$0$'' is $O(n^2)$. Assumption \ref{a1} for other digits can be justified in the same way.
\end{proof}

For $i=1,2,\ldots,n$ and $y\in\{0,1\}^{\otimes n}$, let $F_{i,y}$ be a matrix of dimension $2^{(i-1)n}\times 2^{(n-i)n}$. Identifying a bit string with a binary number, the matrix elements of $F_{i,y}$ are defined as
\begin{equation}
F_{i,y}(1+\overline{x_1x_2\cdots x_{i-1}},1+\overline{x_{i+1}x_{i+2}\cdots x_n})=f(\overline{x_1x_2\cdots x_{i-1}yx_{i+1}\cdots x_n}).
\end{equation}
We add ones on the left-hand side in order to keep the row and column indices of a matrix positive. Indeed, $F_{i,y}$ is obtained by reshaping $f(x)$ after fixing $y$ to be the configuration of $x$ in the $i$th row. The second assumption states that the rank (i.e., the number of nonzero singular values) of $F_{i,y}$ is small.

\begin{assumption} \label{a2}
\begin{equation}
\rank F_{i,y}=O(1).
\end{equation}
\end{assumption}

This assumption can be relaxed to $\rank F_{i,y}=\poly n$ without affecting any result in this paper.

\begin{proof}[Justification of Assumption \ref{a2}]
Consider again the images representing ``$0$.'' If the row configuration $y$ does not belong to any of the three cases listed above, then $F_{i,y}$ is a zero matrix and has rank $0$.

If $y$ is in \emph{Case 1}, the digit is either above or below the $i$th row. Let $g_1$ be a row vector of length $(i-1)n$. The entries are given by $g_1(1+\overline{x_1x_2\cdots x_{i-1}})=1$, if the string $\overline{x_1x_2\cdots x_{i-1}}$, interpreted as an image with $(i-1)\times n$ pixels, represents ``$0$,'' and $g_1(1+\overline{x_1x_2\cdots x_{i-1}})=0$ otherwise. Similarly, define $g_2$ as a row vector of length $(n-i)n$ such that $g_2$ indicates whether an image with $(n-i)\times n$ pixels represents ``$0$.'' We claim
\begin{equation}
F_{i,y}=g_1^T(1\quad\underbrace{0\quad0\quad\cdots\quad0}_{2^{(n-i)n}-1~{\rm zeros} })+(1\quad\underbrace{0\quad0\quad\cdots\quad0}_{2^{(i-1)n}-1~{\rm zeros}})^Tg_2.
\end{equation}
Indeed, the first term on the right-hand side indicates that the digit ``$0$'' is above the $i$th row, where the row vector $(1~0~0~\cdots~0)$ fixes all pixels below the $i$th row as white. Similarly, the second term indicates that the digit is below the $i$th row. Thus, $\rank F_{i,y}=2$.

If $y$ is in \emph{Case 2}, a similar argument implies $\rank F_{i,y}=2$.

If $y$ is in \emph{Case 3}, the $i$th row divides the digit into two parts. Let $h_1$ be a row vector of length $2^{(i-1)n}$. The entries are given by $h_1(1+\overline{x_1x_2\cdots x_{i-1}})=1$, if the string $\overline{x_1x_2\cdots x_{i-1}}$, interpreted as an image with $(i-1)\times n$ pixels, represents the upper part of ``$0$,'' i.e., an arc-like curve connecting the two connected components of black pixels in $y$ from above. Otherwise, $h_1(1+\overline{x_1x_2\cdots x_{i-1}})=0$. Similarly, define $h_2$ as a row vector of length $2^{(n-i)n}$ such that $h_2$ indicates whether an image with $(n-i)\times n$ pixels represents the lower part of ``$0$.'' We claim
\begin{equation}
F_{i,y}=h_1^Th_2.
\end{equation}
Indeed, any two (sub)images representing the upper and lower parts of ``$0$'' can be pieced together to give a full ``$0$,'' provided that they match at the boundary as specified by $y$. Thus, $\rank F_{i,y}=1$.

Taking into account all the cases, $\rank F_{i,y}\le2$ if $f$ is the indicator function of ``$0$.'' Assumption \ref{a2} for other digits can be justified in the same way. As an additional example, $\rank F_{i,y}=2$ if $f$ is the indicator function of ``$8$'' and if there are two connected components of black pixels in $y$.
\end{proof}

For $i=1,2,\ldots,n-1$, let $F_i$ be a matrix of dimension $2^{in}\times 2^{(n-i)n}$ obtained by reshaping $f(x)$. In particular, the matrix elements of $F_i$ are given by
\begin{equation}
F_i(1+\overline{x_1x_2\cdots x_i},1+\overline{x_{i+1}x_{i+2}\cdots x_n})=f(\overline{x_1x_2\cdots x_{i-1}x_ix_{i+1}\cdots x_n}).
\end{equation}

\begin{lemma} \label{l1}
Assumptions \ref{a1}, \ref{a2} imply
\begin{equation}
\rank F_i=\poly n.
\end{equation}
\end{lemma}

\begin{proof}
The $(1+\overline{x_1x_2\cdots x_i})$th row of $F_i$ is identical to the $(1+\overline{x_1x_2\cdots x_{i-1}})$th row of $F_{i,x_i}$. Hence,
\begin{equation}
\rank F_i\le\sum_{y\in\{0,1\}^{\otimes n}}\rank F_{i,y}=\sum_{y\in\{x_i:f(x)=1\}}\rank F_{i,y}.
\end{equation}
The last summation has only a polynomial number of terms (Assumption \ref{a1}), each of which is $O(1)$ (Assumption \ref{a2}).
\end{proof}

\section{Low-rankness}

In this section, we discuss the implications of Assumptions \ref{a1}, \ref{a2} at a conceptual rather than technical level. To put it in a broader context, we first slightly generalize Lemma \ref{l1}.

Let $A$ be a rectangular region of the image, and $|A|$ be the number of pixels in $A$. Let $\bar A$ be the complement of $A$ so that $|\bar A|=n^2-|A|$. Let $|\partial A|$ be the perimeter (length of the boundary $\partial A$) of $A$. Let $x_A\in\{0,1\}^{\otimes|A|}$ be the configuration of an image $x$ in $A$. We reshape the indicator function $f(x)$ into a matrix $F_A$ of dimension $2^{|A|}\times2^{|\bar A|}$ such that the row (column) index corresponds to $x_A$ ($x_{\bar A}$). In particular, the matrix elements of $F_A$ are given by
\begin{equation}
F_A(1+x_A,1+x_{\bar A})=f(x).
\end{equation}

\begin{corollary} \label{c1}
Slight generalizations of Assumptions \ref{a1}, \ref{a2} imply
\begin{equation} \label{eqc1}
\rank F_A=\poly|\partial A|.
\end{equation}
\end{corollary}

\begin{proof}
This corollary reduces to Lemma \ref{l1} in the special case where $A$ consists of the first $i$ rows of the image. It can be proved in the same way as Lemma \ref{l1} by slightly generalizing Assumptions \ref{a1}, \ref{a2}. For instance, Assumption \ref{a1} should be modified to ``the number of configurations on the boundary is $\poly|\partial A|$.'' It should be clear that the validity of such generalizations of Assumptions \ref{a1}, \ref{a2} is still very intuitive in many image classification problems.
\end{proof}

If $f(x)$ were a function such that $f^{-1}(1)$ is a random subset of $\{0,1\}^{\otimes n^2}$, then for $|A|\le n^2/2$,
\begin{equation}
\rank F_A=e^{\Theta(|A|)}
\end{equation}
with overwhelming probability. Thus, Eq. (\ref{eqc1}) is a mathematical characterization of the fact that the indicator function $f(x)$ is highly structured.

The low-rank condition (\ref{eqc1}) provides an analytical understanding of ``local feature,'' an intuitive concept explaining the working principle of neural networks. A local feature is a particular pattern shared by many different (sub)images on a patch, and local features can be pieced together to form objects. The matrix $F_A$ can be decomposed as
\begin{equation}
F_A=\sum_{i=1}^{\poly|\partial A|}g_{A,i}^Tg_{\bar A,i},
\end{equation}
where each $g_{A,i}(g_{\bar A,i})$ is a (provably sparse) row vector of length $2^{|A|}(2^{|\bar A|})$. One might interpret $g_{A,i}$ as the indicator function of the $i$th local feature in $A$. (Here we neglect the technical subtlety that $g_{A,i}$ may not be a $0$-$1$ vector.) The product $g_{A,i}^Tg_{\bar A,i}$ specifies a rule for combining a local feature in $A$ with one in $\bar A$. The number of such combinations is determined by the number of configurations on the boundary $\partial A$, for the reason that the configurations in $A$ and $\bar A$ must match on $\partial A$.

\section{Neural network representation} \label{s4}

In this section, we prove that the low-rank condition established in the preceding section implies an efficient representation of the indicator function $f(x)$ using a particular type of (deep) convolutional arithmetic circuits (ConvAC) \cite{CSS16}. ConvAC is a variant of the more common convolutional neural networks (ConvNet) with rectified linear units (ReLU). The former uses a different set of activation and pooling functions, but is otherwise similar to the latter. Ref. \cite{CS16} provides a detailed discussion of the relationship between ConvAC and ConvNet. A nontechnical introduction to ConvAC is \cite{CSL+17}.

We now describe the neural network in detail. Assume without loss of generality that $\log_2n$ is an integer. Consider a complete binary tree with $2\log_2n+1$ layers, where the $i$th layer (from leaf to root) has $n^2/2^{i-1}$ nodes. Each node is indexed by three integers $(i,j,k)$. The first is the layer index, and the other two specify the spatial location (to be defined more precisely later) of the node. Let $[\cdots]$ denote the floor function. In the $i$th layer, the indices $j$ and $k$ run from $1$ to $n/2^{[i/2]}$ and $n/2^{[(i-1)/2]}$, respectively. The two children of the node $(i,j,k)$ are $(i-1,2j-1,k),(i-1,2j,k)$ for even $i$ and $(i-1,j,2k-1),(i-1,j,2k)$ for odd $i\ge3$.

Each node $(i,j,k)$ carries a function. The function outputs a row vector, whose length, denoted by $l_i$, is called the number of channels in the $i$th layer. We fix $l_{2\log_2n+1}=1$ so that the output of the root node is a number. For a leaf node $(1,j,k)$, the output is the (basis) vector $(1~0)$ or $(0~1)$ if the pixel in the $j$th row and $k$th column is black or white, respectively. The function carried by a non-leaf node $(i,j,k)$ is a composition of convolution, trivial activation, and product pooling. (In comparison, ConvNet uses ReLU activation and max/average pooling.) Taking the outputs $u,v$ of the two children as input, the $m$th entry in the output of the function is given by 
\begin{equation} \label{vec}
\xrightarrow{\rm input}u,v\xrightarrow{\rm 2\to1~pooling}u\odot v\xrightarrow{\rm 1\times1~convolution}V_{i,j,k,m}(u\odot v)^T\xrightarrow{\rm trivial~activation}V_{i,j,k,m}(u\odot v)^T.
\end{equation}
Here, $\odot$ denotes element-wise multiplication, and $V_{i,j,k,m}$, whose entries are the tuning parameters of the network, is a row vector of length $l_{i-1}$.

The network has a clear causal structure. Identifying a leaf node $(1,j,k)$ with the pixel at the position $(j,k)$, the support $S_{i,j,k}$ of a node $(i,j,k)$ is defined as the region of all pixels that are descendants of the node. This definition implies that (i) the output of a node depends only on the pixels in its support; (ii) the supports of the nodes in the same layer are pairwise disjoint; (iii) the support of a node can be obtained by merging the supports of its two children; (iv) the support of every node in an odd layer is a square and that in an even layer is a rectangle with aspect ratio $2$. These properties suggest interpreting the working principle of the network as follows. Let $p(i,j,k)$ and $s(i,j,k)$ be the parent and sibling of the node $(i,j,k)$, respectively. The local features of the region $S_{i,j,k}$ are extracted by the subtree rooted at $(i,j,k)$, and are encoded in the output of $(i,j,k)$. This is because the output of any other node in the $i$th layer does not depend on the pixels in $S_{i,j,k}$ at all. The function carried by $p(i,j,k)$ combines the local features of the regions $S_{i,j,k}$ and $S_{s(i,j,k)}$ in a particular way to obtain the local features of $S_{p(i,j,k)}=S_{(i,j,k)}\cup S_{s(i,j,k)}$ at a larger scale.

\begin{theorem} \label{t1}
Corollary \ref{c1} implies an exact representation of $f(x)$ using ConvAC as described above, where the number of channels in the $i$th layer is upper bounded by
\begin{equation}
l_i=2^{O(i)}.
\end{equation}
\end{theorem}

This theorem (indirectly) supports the conventional wisdom that the number of channels should grow with the layer index in ConvNet.

\begin{proof} [Proof of Theorem \ref{t1}]
We first slightly generalize the architecture by allowing different channels to mix in pooling. In particular, we modify Eq. (\ref{vec}) to
\begin{equation} \label{mat}
\xrightarrow{\rm input}u,v\xrightarrow{\rm 2\to1~pooling}u^Tv\xrightarrow{\rm 1\times1~convolution}\tr(M_{i,j,k,m}u^Tv)\xrightarrow{\rm trivial~activation}vM_{i,j,k,m}u^T,
\end{equation}
where $M_{i,j,k,m}$ is a matrix of dimension $l_{i-1}\times l_{i-1}$. It should be clear that Eq. (\ref{vec}) corresponds to the special case that $M_{i,j,k,m}=\diag V_{i,j,k,m}$ is diagonal. Such a generalized ConvAC is also known as hierarchical Tucker format \cite{HK09, OT09} in applied mathematics and tree tensor network \cite{SDV06} in physics. Thus, we may import existing results in the literature to upper bound the number of channels.
\begin{lemma} [\cite{SDV06, Gra10}]
The function $f(x)$ allows an exact representation in hierarchical Tucker format, provided that
\begin{equation}
l_i\ge\rank F_{S_{i,j,k}},\quad\forall j,k.
\end{equation}
\end{lemma}
Corollary \ref{c1} implies
\begin{equation}
\rank F_{S_{i,j,k}}=\poly|\partial S_{i,j,k}|=2^{O(i)}.
\end{equation}
We complete the proof of Theorem \ref{t1} by showing that the matrix $M_{i,j,k,m}$ in Eq. (\ref{mat}) can be brought into diagonal form at the price of squaring the number of channels in every layer. In the simplest case of two channels, suppose
\begin{equation}
u=(u_1~u_2),\quad v=(v_1~v_2),\quad M_{i,j,k,m}=\begin{pmatrix}a&b\\c&d\end{pmatrix}.
\end{equation}
We copy each entry in $u$ and $v$ once in an appropriate order so that the outputs of the two children become $(u_1~u_2~u_1~u_2)$ and $(v_1~v_1~v_2~v_2)$, respectively. Then,
\begin{equation}
vM_{i,j,k,m}u^T=(v_1~v_1~v_2~v_2)\diag(a,b,c,d)(u_1~u_2~u_1~u_2)^T.
\end{equation}
It should be clear that the general case of more than two channels can be handled similarly.
\end{proof}

The architecture considered in this section has the property that the supports of the nodes in the same layer are pairwise disjoint. This property is desirable for interpreting the functionality of each node, and a network with the non-overlapping property (and a polynomial number of channels) is sufficiently powerful to represent the function $f(x)$. However, we believe that it may be practically beneficial (in the sense of reducing the number of channels) to have overlapping supports for nodes in the same layer, as in ``standard'' ConvNet.

Another approach that may improve the practical performance of hierarchical Tucker format is to use the multiscale entanglement renormalization ansatz \cite{Vid07}.

\section{Tensor train format}

In this section, we prove that Assumptions \ref{a1}, \ref{a2} imply an efficient representation of the indicator function $f(x)$ in the tensor train format \cite{Ose11}, which is also known as matrix product representation \cite{PVWC07} in physics.

We fix an ordering of the pixels by indexing the pixel in the $i$th row and $j$th column with the integer $k:=(i-1)n+j$. We associate each pixel with two matrices $M_{k,0},M_{k,1}$ of dimension $l_{k-1}\times l_k$. The tensor train format encodes a function, which maps an image $y=\overline{y_1y_2\cdots y_{n^2}}\in\{0,1\}^{\otimes n^2}$ (note that $y_k\in\{0,1\}$ represents the $k$th pixel) to
\begin{equation}
M_{1,y_1}M_{2,y_2}\cdots M_{n^2,y_{n^2}},
\end{equation}
where we fix $l_0=l_{n^2}=1$ so that the output is a number.

\begin{lemma} [\cite{Vid03, Ose11}]
The function $f(x)$ allows an exact representation in the tensor train format, provided that
\begin{equation}
l_k\ge\rank F_{B_k},
\end{equation}
where $B_k$ is the region of pixels with indices from $1$ to $k$.
\end{lemma}

\begin{theorem} \label{t2}
Assumptions \ref{a1}, \ref{a2} imply an exact representation of $f(x)$ using the tensor train format, where the matrix dimension is upper bounded by
\begin{equation}
\max_kl_k=\poly n.
\end{equation}
\end{theorem}

\begin{proof}
It suffices to prove $\rank F_{B_k}=\poly n$. However, we cannot directly use Corollary \ref{c1}, because $B_k$ may not be rectangular. Suppose the pixel labeled by $k$ is in the $i$th row and $j$th column such that $k=(i-1)n+j$. To simplify the notation, let $Y_1=\overline{y_{(i-1)n+1}y_{(i-1)n+2}\cdots y_k}$ and $Y_2=\overline{y_{k+1}y_{k+2}\cdots y_{in}}$. It is straightforward to see that the matrix $F_{B_k}$ can be partitioned into blocks indexed by $Y_1,Y_2$:
\begin{equation}
F_{B_k}=(F_{B_k,Y_1,Y_2})_{Y_1\in\{0,1\}^{\otimes j},Y_2\in\{0,1\}^{\otimes(n-j)}},
\end{equation}
where the submatrix $F_{B_k,Y_1,Y_2}$ is of dimension $2^{(i-1)n}\times2^{(n-i)n}$ with matrix elements given by
\begin{equation}
F_{B_k,Y_1,Y_2}(1+\overline{y_1y_2\cdots y_{(i-1)n}},1+\overline{y_{in+1}y_{in+2}\cdots y_{n^2}})=f(\overline{y_1y_2\cdots y_{(i-1)n}Y_1Y_2y_{in+1}\cdots y_{n^2}}).
\end{equation}
Hence,
\begin{equation}
\rank F_{B_k}\le\sum_{Y_1\in\{0,1\}^{\otimes j},Y_2\in\{0,1\}^{\otimes(n-j)}}\rank F_{B_k,Y_1,Y_2}=\sum_{\substack{Y_1\in\{0,1\}^{\otimes j},Y_2\in\{0,1\}^{\otimes(n-j)}\\\overline{Y_1Y_2}\rm~valid~row~configuration}}\rank F_{B_k,Y_1,Y_2}.
\end{equation}
The last summation has only a polynomial number of terms (Assumption \ref{a1}), each of which is $O(1)$ (Assumption \ref{a2}).
\end{proof}

\section{Discussions}

The key insight of our approach is the low-rankness (Corollary \ref{c1}) of $F_A$, which is a matrix obtained by reshaping the function $f(x)$ such that the row (column) index corresponds to the configuration of the region $A$ ($\bar A$). This way of thinking is standard practice in quantum physics, and was adopted in a series of recent papers \cite{DLD17, CCX+17, LYCS17, Zha17} to characterize the expressive power of neural networks. In particular, Ref. \cite{Zha17} proved the so-called ``area law''
\begin{equation} \label{al}
\rank F_A=e^{O(|\partial A|)}
\end{equation}
based on assumptions that are conceptually similar to but technically different from those in Sec. \ref{s2}. Ref. \cite{DLD17} showed that any function represented by a restricted Boltzmann machine with short-range interactions obeys the area law. It should be noted that the converse is not true. Let $S$ be the set of all functions obeying the area law. A simple counting argument implies that for any architecture, a random element in $S$ does not have an efficient representation with overwhelming probability \cite{GE16}. In contrast, we have proved that Corollary \ref{c1}, an exponential improvement on the area law, implies an efficient neural network representation.

Finally, we comment on the practical performance of the architectures considered in this paper. Ref. \cite{SS16} designed a training algorithm for the tensor train format by adapting existing optimization techniques \cite{Whi92, Sch11, HRS12}. For the MNIST database of handwritten digits, test error rates of $2\%,0.97\%$ can be achieved with matrix dimensions $\max_kl_k=20,120$, respectively. Preliminary results for the CIFAR-10 dataset were reported in a very recent paper \cite{KNO17}. Ref. \cite{HWF+17} proposed a generative model based on the tensor train format. Ref. \cite{NPOV16} used the tensor train format to compress fully-connected layers in a deep neural network.

Ref. \cite{LRW+17} designed a training algorithm for hierarchical Tucker format by importing optimization techniques from the physics literature \cite{EV09}. For MNIST, a test error rate of $4\%$ can be achieved with $l_k=3$ channels in every layer. Preliminary results for CIFAR-10 were reported in the same paper. It should be noted that the architecture used in \cite{LRW+17} is almost identical to the one described in Sec. \ref{s4}, except that each non-leaf node in the tree has $4$ children.

In these experiments on MNIST, the number of channels in hierarchical Tucker format and the matrix dimension in the tensor train format appear to be smaller than those in Theorems \ref{t1} and \ref{t2}. This is because the digits in MNIST have been size-normalized and centered. Such preprocessing greatly reduces the number of row configurations so that the estimate in Assumption \ref{a1} is far from tight. We expect that the scaling relation in Lemma \ref{l1} would be observed in an experiment where the digits were randomly sized and position-shifted. We leave it to future work.

\section*{Acknowledgments}

We would like to thank Tengyu Ma and Yang Yuan for helpful comments, and especially Pengchuan Zhang for inspiring discussions.

We acknowledge funding provided by the Institute for Quantum Information and Matter, an NSF Physics Frontiers Center (NSF Grant PHY-1125565) with support of the Gordon and Betty Moore Foundation (GBMF-2644). Additional funding support was provided by NSF DMR-1654340.

\bibliographystyle{abbrv}
\bibliography{ml}

\begin{thebibliography}{10}

\bibitem{CCX+17}
J.~Chen, S.~Cheng, H.~Xie, L.~Wang, and T.~Xiang.
\newblock On the equivalence of restricted {B}oltzmann machines and tensor
  network states.
\newblock {a}rXiv:1701.04831, 2017.

\bibitem{CSL+17}
N.~Cohen, O.~Sharir, Y.~Levine, R.~Tamari, D.~Yakira, and A.~Shashua.
\newblock Analysis and design of convolutional networks via hierarchical tensor
  decompositions.
\newblock {a}rXiv:1705.02302, 2017.

\bibitem{CSS16}
N.~Cohen, O.~Sharir, and A.~Shashua.
\newblock On the expressive power of deep learning: A tensor analysis.
\newblock In {\em 29th Annual Conference on Learning Theory}, pages 698--728,
  2016.

\bibitem{CS16}
N.~Cohen and A.~Shashua.
\newblock Convolutional rectifier networks as generalized tensor
  decompositions.
\newblock In {\em 33rd International Conference on Machine Learning}, pages
  955--963, 2016.

\bibitem{DLD17}
D.-L. Deng, X.~Li, and S.~Das~Sarma.
\newblock Quantum entanglement in neural network states.
\newblock {\em Physical Review X}, 7(2):021021, 2017.

\bibitem{EV09}
G.~Evenbly and G.~Vidal.
\newblock Algorithms for entanglement renormalization.
\newblock {\em Physical Review B}, 79(14):144108, 2009.

\bibitem{GE16}
Y.~Ge and J.~Eisert.
\newblock Area laws and efficient descriptions of quantum many-body states.
\newblock {\em New Journal of Physics}, 18(8):083026, 2016.

\bibitem{Gra10}
L.~Grasedyck.
\newblock Hierarchical singular value decomposition of tensors.
\newblock {\em SIAM Journal on Matrix Analysis and Applications},
  31(4):2029--2054, 2010.

\bibitem{HK09}
W.~Hackbusch and S.~K{\"u}hn.
\newblock A new scheme for the tensor representation.
\newblock {\em Journal of Fourier Analysis and Applications}, 15(5):706--722,
  2009.

\bibitem{HWF+17}
Z.-Y. Han, J.~Wang, H.~Fan, L.~Wang, and P.~Zhang.
\newblock Unsupervised generative modeling using matrix product states.
\newblock {a}rXiv:1709.01662, 2017.

\bibitem{HRS12}
S.~Holtz, T.~Rohwedder, and R.~Schneider.
\newblock The alternating linear scheme for tensor optimization in the tensor
  train format.
\newblock {\em SIAM Journal on Scientific Computing}, 34(2):A683--A713, 2012.

\bibitem{KNO17}
V.~Khrulkov, A.~Novikov, and I.~Oseledets.
\newblock Expressive power of recurrent neural networks.
\newblock {a}rXiv:1711.00811, 2017.

\bibitem{LBBH98}
Y.~Lecun, L.~Bottou, Y.~Bengio, and P.~Haffner.
\newblock Gradient-based learning applied to document recognition.
\newblock {\em Proceedings of the IEEE}, 86(11):2278--2324, 1998.

\bibitem{LYCS17}
Y.~Levine, D.~Yakira, N.~Cohen, and A.~Shashua.
\newblock Deep learning and quantum entanglement: Fundamental connections with
  implications to network design.
\newblock {a}rXiv:1704.01552, 2017.

\bibitem{LTR17}
H.~W. Lin, M.~Tegmark, and D.~Rolnick.
\newblock Why does deep and cheap learning work so well?
\newblock {\em Journal of Statistical Physics}, 168(6):1223--1247, 2017.

\bibitem{LRW+17}
D.~Liu, S.-J. Ran, P.~Wittek, C.~Peng, R.~B. Garc{\'i}a, G.~Su, and
  M.~Lewenstein.
\newblock Machine learning by two-dimensional hierarchical tensor networks: A
  quantum information theoretic perspective on deep architectures.
\newblock {a}rXiv:1710.04833, 2017.

\bibitem{NPOV16}
A.~Novikov, D.~Podoprikhin, A.~Osokin, and D.~P. Vetrov.
\newblock Tensorizing neural networks.
\newblock In {\em Advances in Neural Information Processing Systems 28}, pages
  442--450, 2015.

\bibitem{Ose11}
I.~V. Oseledets.
\newblock Tensor-train decomposition.
\newblock {\em SIAM Journal on Scientific Computing}, 33(5):2295--2317, 2011.

\bibitem{OT09}
I.~V. Oseledets and E.~E. Tyrtyshnikov.
\newblock Breaking the curse of dimensionality, or how to use {SVD} in many
  dimensions.
\newblock {\em SIAM Journal on Scientific Computing}, 31(5):3744--3759, 2009.

\bibitem{PVWC07}
D.~Perez-Garcia, F.~Verstraete, M.~M. Wolf, and J.~I. Cirac.
\newblock Matrix product state representations.
\newblock {\em Quantum Information and Computation}, 7(5-6):401--430, 2007.

\bibitem{Sch11}
U.~Schollw{\"o}ck.
\newblock The density-matrix renormalization group in the age of matrix product
  states.
\newblock {\em Annals of Physics}, 326(1):96--192, 2011.

\bibitem{SDV06}
Y.-Y. Shi, L.-M. Duan, and G.~Vidal.
\newblock Classical simulation of quantum many-body systems with a tree tensor
  network.
\newblock {\em Physical Review A}, 74(2):022320, 2006.

\bibitem{SS16}
E.~Stoudenmire and D.~J. Schwab.
\newblock Supervised learning with tensor networks.
\newblock In {\em Advances in Neural Information Processing Systems 29}, pages
  4799--4807, 2016.

\bibitem{Vid03}
G.~Vidal.
\newblock Efficient classical simulation of slightly entangled quantum
  computations.
\newblock {\em Physical Review Letters}, 91(14):147902, 2003.

\bibitem{Vid07}
G.~Vidal.
\newblock Entanglement renormalization.
\newblock {\em Physical Review Letters}, 99(22):220405, 2007.

\bibitem{Whi92}
S.~R. White.
\newblock Density matrix formulation for quantum renormalization groups.
\newblock {\em Physical Review Letters}, 69(19):2863--2866, 1992.

\bibitem{Zha17}
Y.-H. Zhang.
\newblock Entanglement entropy of target functions for image classification and
  convolutional neural network.
\newblock {a}rXiv:1710.05520, 2017.

\end{thebibliography}

\end{document}